\documentclass[11pt, a4paper]{article}

\usepackage[a4paper,margin=2.5cm]{geometry}
\usepackage[utf8]{inputenc}
\usepackage{booktabs}
\usepackage[small]{caption}
\usepackage{lmodern}
\usepackage{amsthm}
\usepackage{subcaption}
\usepackage{tikz}
\usepackage{algorithm}
\usepackage{algpseudocode}
\usepackage{xspace}
\usepackage{mathtools}
\usepackage[super]{nth}
\usepackage{thm-restate}
\usepackage{microtype}
\usepackage{hyperref}
\usepackage{enumitem}
\usepackage{multirow}
\usepackage{amsfonts}
\usepackage{svg}
\usepackage[sort&compress,numbers]{natbib}
\usepackage{cleveref}
\usepackage[font=scriptsize]{caption}
\usepackage{comment}

\usepackage{cleveref}
\usepackage{makecell}
\usepackage{pifont}
\usepackage{colortbl}
\usepackage{footmisc}
\allowdisplaybreaks 

\algdef{SE}[SUBALG]{Indent}{EndIndent}{}{\algorithmicend\ }%
\algtext*{Indent}
\algtext*{EndIndent}

\algdef{SE}{Upon}{EndUpon}[1]{\textbf{upon} \(\mbox{#1}\) \textbf{do}}{\textbf{end upon}}%

\theoremstyle{plain}
\newtheorem{theorem}{Theorem}[section]
\newtheorem{definition}[theorem]{Definition}
\newtheorem{lemma}[theorem]{Lemma}

\theoremstyle{remark}

\newtheorem{notation}{Notation}

\newenvironment{sloppypar*}
 {\sloppy\ignorespaces}
 {\par}

\newcommand{\qedClaim}{\hfill \ensuremath{\Box}}

\newcommand{\fedavg}{\textit{FedAvg}\xspace}
\newcommand{\fedsgd}{\textit{FedSGD}\xspace}

\newcommand{\barSet}{\ensuremath{\mathrm{S_{Cent}}\xspace}}
\newcommand{\trueBar}{\ensuremath{\mathrm{Cent}^{\star}}\xspace}
\newcommand{\encBall}{\ensuremath{\mathrm{Ball_{cov}(\barSet)}}\xspace}

\newcommand{\radiusEncBall}{\ensuremath{\mathrm{Rad_{cov}}}\xspace}

\newcommand{\convexHull}{\ensuremath{\mathrm{Conv}}\xspace}
\newcommand{\distance}{\ensuremath{\mathrm{dist}}\xspace}
\newcommand{\tb}{\ensuremath{\textrm{TH}}\xspace}
\newcommand{\ttb}{\ensuremath{\textrm{TTH}}\xspace}
\newcommand{\bb}{\ensuremath{\textrm{CH}}\xspace}

\DeclareMathOperator*{\argmin}{arg\,min}

\newcommand\restr[2]{{
\left.\kern-\nulldelimiterspace 
#1 
\vphantom{\big|} 
\right|_{#2} 
}}

\newcommand{\mytitle}[1]{
    \begingroup
    \fontsize{16pt}{18pt}\fontseries{bx}\selectfont \centering #1 \par  
    \endgroup
}
\newcommand{\myauthors}[1]{
    \begingroup
    \fontsize{12pt}{14pt}\fontseries{bx}\selectfont \centering #1 \par  
    \endgroup
}

\begin{document}

\mytitle{Centroid Approximation for \\Byzantine-Tolerant Federated Learning}

\bigskip
\bigskip

\newcounter{fnnumber}

\myauthors{Melanie Cambus\footnote{Aalto University, Finland}, Darya Melnyk\footnote{TU Berlin, Germany}\setcounter{fnnumber}{\thefootnote}  , Tijana Milentijević\footnotemark[\thefnnumber], Stefan Schmid\footnotemark[\thefnnumber]
}

\bigskip
\bigskip

\begin{abstract}
Federated learning allows each client to keep its data locally when training machine learning models in a distributed setting. Significant recent research established the requirements that the input must satisfy in order to guarantee convergence of the training loop. 
This line of work uses averaging as the aggregation rule for the training models. 
In particular, we are interested in whether federated learning is robust to Byzantine behavior, and observe and investigate a tradeoff between the average/centroid and the validity conditions from distributed computing.
We show that the various validity conditions alone do not guarantee a good approximation of the average. 
Furthermore, we show that reaching good approximation does not give good results in experimental settings due to possible Byzantine outliers. 
Our main contribution is the first lower bound of $\min\{\frac{n-t}{t},\sqrt{d}\}$ on the centroid approximation under box validity that is often considered in the literature, where $n$ is the number of clients, $t$ the upper bound on the number of Byzantine faults, and $d$ is the dimension of the machine learning model. We complement this lower bound by an upper bound of $2\min\{n,\sqrt{d}\}$, by providing a new analysis for the case $n<d$. In addition, we present a new algorithm that achieves a $\sqrt{2d}$-approximation under convex validity, which also proves that the existing lower bound in the literature is tight. We show that all presented bounds can also be achieved in the distributed peer-to-peer setting. We complement our analytical results with empirical evaluations in federated stochastic gradient descent and federated averaging settings.
\end{abstract}

\section{Introduction}

Federated learning~\cite{google-federated, fedavg} is a decentralized technique for training machine learning models based on sharing model parameters while keeping the training data \emph{locally}. In this work, we are particularly interested in the setting where the clients share updates --- namely either the gradients in case of \emph{federated stochastic gradient descent (\fedsgd)} or the model parameters in case of \emph{federated averaging (\fedavg)} --- with a trusted central server.  After the server has received the updates, it aggregates the results, updates the model parameters, and then shares the new model parameters with the clients for the next training round. This technique is popular when data privacy requirements prevent clients from sharing their data directly with the server \cite{ZHANG2021106775,white-house-report, kairouz2021advances}.
The most common aggregation rule used to select a representative vector (gradient or model parameters) is averaging~\cite{google-federated, fedavg, zhao2018federated, adaptivefederatedoptimization, scaffold, fedlin, fednova, fedprox, fedexp, fedvarp}. However, when averaging is used, training can fail if some clients do not behave as expected. In particular, a single faulty vector can arbitrarily shift the average in any direction, leading to erroneous updates of the model parameters. Especially in the context of federated learning, it is crucial to be robust to malicious behavior and Byzantine faults, which is also the focus of our paper. In the case of homogeneous training data, it is usually possible to use similarities between vectors to exclude such outliers~\cite{fang2022bridgebyzantineresilientdecentralizedgradient, Yang_2019, el2020genuinely}. If the data is heterogeneous, such similarities may not exist. 

Previously proposed Byzantine-tolerant FL methods for heterogeneous datasets focus on showing convergence of the training process and apply statistical methods for vector aggregation~\cite{data2021byzantine,10.1609/aaai.v33i01.33011544,ghosh2019robustfederatedlearningheterogeneous}. To mitigate Byzantine behavior, their methods remove outliers from the data and make additional assumptions on the input vectors of the clients. An alternative approach proposed in the literature is to use the absolute distance to the average to evaluate federated learning algorithms~\cite{jungle}. This absolute measure, however, only allows one to analyze the worst-case input setting. The authors showed that the aggregated average can be as bad as the distance between the two furthest input vectors. Recently, a new approximation measure was introduced to estimate the quality of an aggregated average in a Byzantine environment~\cite{centroid-paper} for approximate agreement algorithms. This approximation measure allows one to not only analyze the worst-case input setting, but rather estimate the quality of an algorithm based on the given input distribution. .

In this work, we transfer the idea of approximating the average vector to the traditional FL setting with $n$ clients and one trusted server. In distributed computing, validity conditions are used to restrict an algorithm from terminating on arbitrary inputs. We investigate the trade-off between the validity conditions and the approximation of the average vector for federated learning. 
This allows us to present aggregation algorithms that perform well under different input distributions. 

\subsection{The benefits of average approximation}

In this paper, we consider approximation of the average to evaluate the quality of our algorithms. As we motivate in the following, a low average approximation ratio implies that an algorithm performs well for a given input distribution. Formally, given $n$ vectors, up to $t$ of which can be Byzantine, an optimal choice of the average vector under Byzantine attacks is defined as the midpoint of the smallest ball, denoted $B$, that encloses each average obtained from every subset of $n-t$ vectors. When $t$ clients are Byzantine, exactly one of these averages was computed from only non-faulty vectors. Therefore, the midpoint minimizes the maximum distance to the non-faulty average vector in the worst case. The approximation ratio is then defined as the ratio between the distance from the aggregation vector to the non-faulty average, and the radius of $B$.

The main advantage of this approximation ratio is that it is defined relative to the input setting:
In scenarios with heterogeneous training data, Byzantine vectors cannot be differentiated from non-faulty vectors. That is, a large radius of the minimum covering  ball either represents ``bad'' Byzantine behavior, or a ``bad'' initial configuration where each client has vastly different input. In such a scenario, no aggregation algorithm can choose a representative average vector. The large ball radius prevents one from punishing an algorithm for a large absolute distance to the average vector. A small radius, on the other hand, represents ``benign'' Byzantine behavior and very similar inputs. In such a scenario, an aggregation algorithm should be able to choose an aggregation vector that is close to the original average. Figure~\ref{fig:approximation_def} visualizes the continuous change in the ball radius depending on the input vectors of the clients. 

\begin{figure}[tb]
\centering

    \begin{subfigure}[b]{0.6\textwidth}            
            \includegraphics[height=4cm]{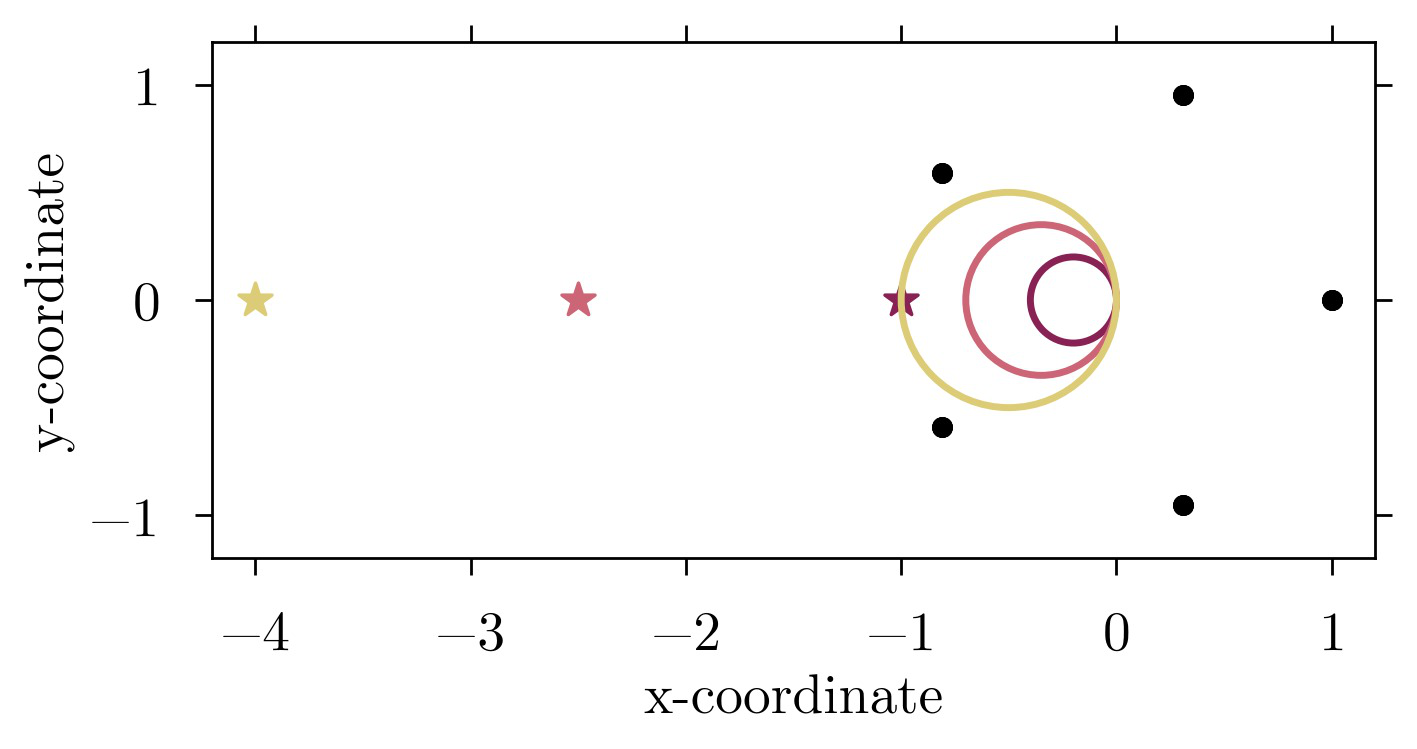}
            \caption{Distribution of the input vectors}
            
    \end{subfigure}
    \begin{subfigure}[b]{0.39\textwidth}
            \centering
            \includegraphics[height=3.9cm]{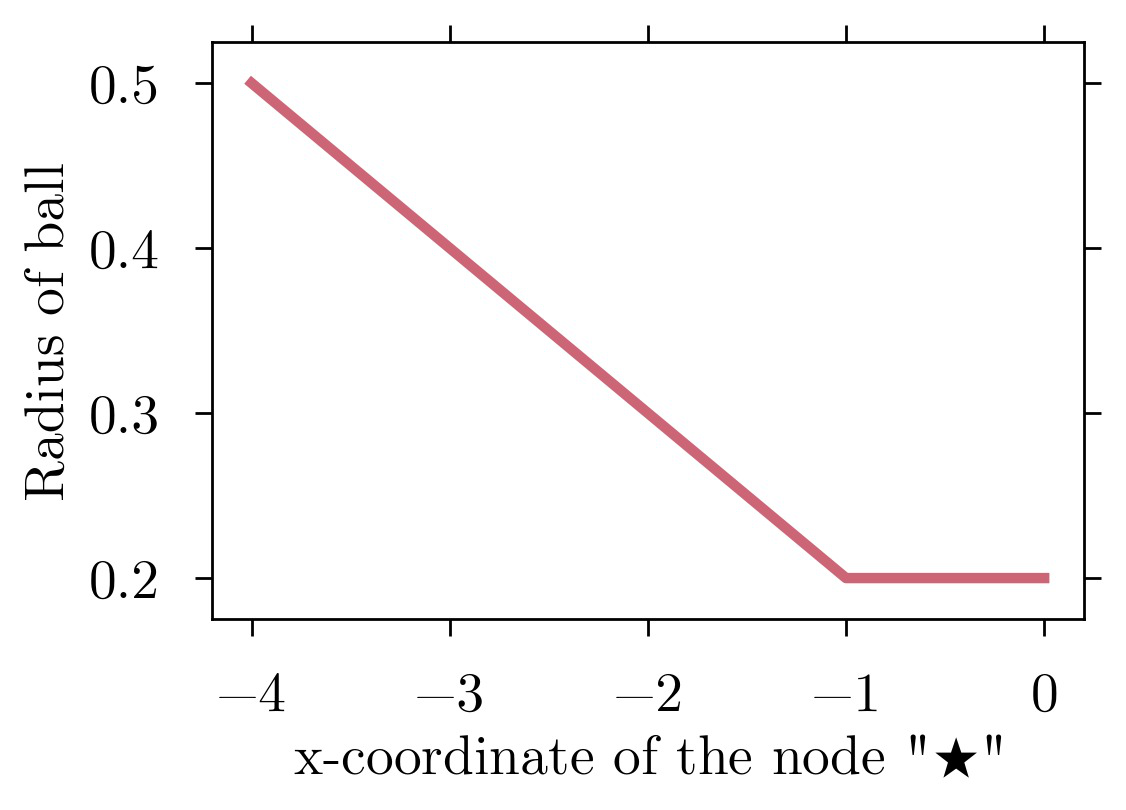}
            \caption{Radius of the ball of averages}
    \end{subfigure}
    \caption{\textbf{Illustration of the benefits of the approximation definition that is relative to the input distribution.} This figure shows how the radius of the smallest ball containing all averages depends on different distributions of the inputs. There are $6$ clients, one of which is possibly Byzantine. On the left, three input scenarios are considered. The points represent input vectors that are fixed in all scenarios. The three stars represent three different inputs of the sixth client. The circles represent the three smallest balls containing all possible averages on subsets of five points, one for each input of the last client. On the right, the radius of the minimum covering  ball is presented when the $x$-coordinate of the last client is moved from $-4$ to $0$. Observe that the radius of the minimum covering  ball cannot be zero, as any of the points in the figure are also potentially Byzantine. The yellow star represents a ``bad'' input setting where it is not clear whether there is one non-faulty client with very distinct data, or whether a Byzantine party tries to disrupt the training process. The dark red star shows a benign setting where an aggregation algorithm should be able to output a vector close to the actual centroid. The radius of the ball used to define the approximation ratio decreases with more benign input distributions.} 
    \label{fig:approximation_def}
\end{figure}

\subsection{Contributions}\label{sec:contributions}

We first show that known validity conditions from the literature do not guarantee good approximation of the average. We then show that under weak and strong validity conditions, both of which only require the server to output the same vector as the non-faulty client if all non-faulty clients send the server the same vector, a constant approximation of the average can be achieved.

Our first main contribution is almost tight bounds for algorithms that satisfy box validity, where the aggregation vector lies in the coordinate-parallel hyperbox of non-faulty vectors. We present a lower bound of $\min\{\sqrt{(n-t) /t},\sqrt{d}\}$ for the centroid approximation and show that the existing Box algorithm can achieve an approximation of $2\sqrt{\min\{n,d\}}$ by providing a new upper bound proof for the case $n<d$. 
Our second main contribution is a tight upper bound (a $2d$-approximation) for convex validity, where the aggregation vector lies in the convex hull of all vectors. Note that this setting is only of theoretical interest, as it requires the number of clients to be larger than the dimension of their input vectors ($n>(d+1)\cdot t$), while in FL the dimension of the data is usually much larger than the number of clients. We show that all presented bounds can also be achieved in the distributed peer-to-peer setting. 
The agreement algorithms presented differ from \cite{jungle,centroid-paper}, since only exact agreement is considered in this paper. 

Finally, we extend our analytical results with simulations. In this evaluation, we differentiate between the settings where the gradients (\fedsgd) and the model parameters (\fedavg) are aggregated, and show how selected algorithms perform under different failure scenarios.

\subsection{Related Work}

Dean et al.~\cite{dean2012largescale} proposed a first distributed solution to train a large machine learning model on tens of thousands of CPU cores. Their work initiated a study of asynchronous algorithms for distributed stochastic gradient descent (SGD) that focus on scalability and communication efficiency~\cite{10.5555/2685048.2685095,NIPS2014_1ff1de77,NIPS2013_d6ef5f7f,pmlr-v32-shamir14}. The synchronous version of SGD has been proposed by Chen et al.~\cite{45187}. We refer to this framework as \fedsgd. \fedsgd has also been considered under Byzantine adversaries, both in synchronous~\cite{Alistarh_ByzantineSGD,pmlr-v80-mhamdi18a} and asynchronous settings~\cite{pmlr-v80-damaskinos18a}. While the mentioned work assumes homogeneous data distributions, some efforts have also been made to incorporate data heterogeneity~\cite{10.1609/aaai.v33i01.33011544,pmlr-v97-xie19b,ghosh2019robustfederatedlearningheterogeneous,data2021byzantine}. To tackle Byzantine behavior of the clients, these approaches make use of homogeneity of the data, apply statistical methods, or try to detect Byzantine behavior.

Federated averaging was introduced by McMahan et al.~\cite{google-federated, fedavg} to perform training where the data is private, unbalanced, non-IID, and distributed across mobile devices. Here, model parameters instead of gradients are exchanged with a server. We refer to this framework as \fedavg in this paper. Much of the follow-up work has focused on showing convergence of the models in this framework without failures~\cite{fedlin,adaptivefederatedoptimization,fednova, fedexp, fedvarp, cho2020clientselectionfederatedlearning}. Byzantine-tolerant approaches have been introduced also for this setting, where the goal is to remove Byzantine behavior via stochastic quantization and outlier detection mechanisms~\cite{brea2021}. 

In contrast to previous work, we do not focus on removing Byzantine clients from the training process, as such a process may influence the accuracy when the data is heterogeneous and no malicious behavior is present in the system. Instead, we use the approximation definition for the average from Cambus and Melnyk~\cite{centroid-paper} that naturally incorporates Byzantine clients. In contrast to \cite{centroid-paper}, we consider a stronger model without agreement, which makes our lower bound results more powerful, and introduce new algorithms that achieve an optimal approximation.

\section{Model and Definitions}
\label{sec:model}

We consider a client/server setting with one server and $n$ clients.
The goal is to train a global neural network on the server with data spread heterogeneously among clients. In order to train the global model without gathering data from clients, each client possesses its own copy of the model and then shares only vectors generated from their local data and model with the server. The server then needs to aggregate the received vectors to advance the training of the global model. The training process is performed in synchronous rounds.

On top of the training set-up, we consider that up to $t<n/3$ of the clients can be Byzantine, i.e., they can behave arbitrarily and are not bound to following the protocol. The aggregation algorithms used by the server hence need to account for this. Note that we use the standard assumption from distributed computing that Byzantine clients are not differentiable from non-faulty clients as long as they follow the protocol and only lie about their input. 
We treat all clients equally, assuming no size difference in the local data, in order to restrict the power of the Byzantine clients.

The focus of this work is on the aggregation function. Consider a specific communication round, in which each client sends a vector to the server, and the server aggregates those vectors. 
To account for the potential presence of Byzantine clients in the system, the aggregation algorithm used by the server needs to compute an aggregation that is as little influenced by Byzantine vectors as possible. In this work, we focus on the most common aggregation rule in FL -- the averaging aggregation rule. 
Since Byzantine clients can be present in the system and are undetectable, it is impossible for an aggregation algorithm to determine the centroid of vectors of non-faulty clients. We are therefore interested in the quality of the computed aggregated vector.

\subsection{Centroid approximation}

We assume that the server receives up to $m$ vectors $\{v_i, i\in[m]\}$, where $n-t\le m \le n$. Each vector $v$ is in the normed vector space $\left(\mathbb{R}^d, \Vert\cdot \Vert_2\right)$, where $\forall x = (x_1, \dots, x_d)\in\mathbb{R}^d, \Vert x\Vert_2 = \sqrt{\sum_{k=1}^dx_k^2 }$ and the distance between any two vectors $v$ and $w$ is their Euclidean distance $\distance(v, w) = \Vert v-w \Vert_2$. When not specified, $\Vert \cdot \Vert$ refers to the $2$-norm. 
We use the following definition of the average/centroid: 
\begin{definition}[Centroid]
    The centroid of a finite set of $k$ vectors $\{v_i, i\in[k]\}$ is $\frac{1}{k}\sum_{i=1}^k v_i$. 
\end{definition}

We define the centroid approximation as in~\cite{centroid-paper}. Let $\trueBar$ be the centroid computed from non-faulty vectors only. Note that there can be up to $n$ non-faulty vectors as $t$ is only an upper bound on the number of Byzantine clients. In the following, we define the set of candidate centroids, which are computed based on the worst case where exactly $t$ vectors are Byzantine.

\begin{definition}[Set of candidate centroids]
    The set $\barSet$ containing all centroids of $n-t$ input vectors is  defined as 
    \begin{align*}
        \barSet \coloneqq \left\{ \frac{1}{n-t}\sum\nolimits_{i\in I}v_i \Big\vert \forall I\in[n] \  s.t.\ |I|=n-t    \right\}.
    \end{align*}
\end{definition}

Due to the assumption that Byzantine clients are not differentiable from non-faulty clients as long as they follow the protocol, we can only define the centroid approximation based on the worst case where exactly $t$ clients are Byzantine.
We define the point minimizing the maximum distance to all vectors in the set of candidate centroids defined above is the center of the following ball: 

\begin{definition}[Minimum covering ball]
    The minimum covering ball $\encBall$ is the smallest ball containing all vectors in $\barSet$. Its radius is denoted $\radiusEncBall$. 
\end{definition}

Finally, the centroid approximation is defined as follows:

\begin{definition}[Centroid approximation]
    Given an input layout $L=\{v_i, i\in[n]\}$, let $O_{\mathcal{A}}$ be the output of an algorithm $\mathcal{A}$ computing an approximation of the centroid of non-faulty vectors. 
    The approximation ratio of $\mathcal{A}$ given $L$ is the smallest $\alpha$ s.t. 
    \begin{align*}
        \distance(O_{\mathcal{A}}, \trueBar)\leq \alpha \cdot \radiusEncBall.
    \end{align*}
    The algorithm $\mathcal{A}$ is said to compute an $\alpha$-approximation of the centroid if, for all input layout $L$, the approximation ratio of $\mathcal{A}$ given $L$ is upper bounded by $\alpha$.
\end{definition}

In order to compute the approximation ratio of a certain type of algorithms, we need to consider a less restrictive area than the minimum covering ball: 

\begin{definition}[Centroid hyperbox]
    The centroid hyperbox $\bb$ is the smallest coordinate-parallel hyperbox containing $\barSet$. 
\end{definition}

\subsection{Validity conditions}

We noted above that a Byzantine client can shift the centroid of vectors of all clients arbitrarily, and thus it can also shift the midpoint of the minimum covering ball arbitrarily far away from $\trueBar$. Just choosing the center as the centroid approximation might not be sufficient to ensure that we can trust the output of a certain algorithm. 
We hence take inspiration from the distributed agreement algorithms and use validity conditions to get additional guarantees on the output of different algorithms, complementing the guarantees given by the centroid approximation ratio. 

A validity condition is satisfied when the output of an algorithm is guaranteed to be in a specific area, depending only on the input layout. 
In this work, we focus on common validity conditions from the literature:

\begin{definition}[Validity conditions]
    Given are $n$ vectors, up to $t$ of which are Byzantine. An algorithm $\mathcal{A}$ satisfies

        \textbf{weak validity}~\cite{civit2022byzantine,civit2021polygraph,yin2019hotstuff} if, when all clients are non-faulty and all input vectors $v_i$ are equal to a single vector $v$, the output of $\mathcal{A}$ is $v$;
        
        \textbf{strong validity}~\cite{10.1007/BFb0040405,BrachaRB,10.1145/800221.806706} if, when all non-faulty input vectors $v_i$ are equal to a single vector $v$, the output of $\mathcal{A}$ is $v$;
         
        \textbf{box validity}~\cite{centroid-paper,10.1145/5925.5931,box-validity-def} if the output of $\mathcal{A}$ is inside the smallest coordinate-parallel hyperbox containing all non-faulty input vectors (Notation~\ref{notation: trusted hyperbox});
         
        \textbf{convex validity}~\cite{abbas2022centerpoint,multidim-approx-agreement,wang2019computingTverbergPoint} if the output of $\mathcal{A}$ is inside the convex hull of all non-faulty input vectors. 

\end{definition}

\begin{notation}\label{notation: trusted hyperbox}
    The smallest coordinate-parallel hyperbox containing only non-faulty vectors is called the trusted hyperbox and denoted $\tb$. 
\end{notation}
Note that the trusted hyperbox cannot be computed in practice. However, we prove in \Cref{sec: theory} (\Cref{lem: box-val-implies-trimmed-box-agreement}) that an algorithm satisfies the box validity condition if and only if it agrees inside a hyperbox called the trimmed trusted hyperbox ($\ttb$):  

\begin{definition}[Trimmed trusted hyperbox]
    Let $v_1, \dots, v_{m}$ be the received input vectors, where $m$ is the number of received messages. 
    The number of Byzantine values for each coordinate is at most $m-(n-t)$. 
    Denote $\phi: [m] \to [m]$ a bijection s.t. $v_{\phi(j_1)} [k]\leq v_{\phi(j_2)} [k], \forall j_1, j_2\in [m]$. 
    The trimmed trusted hyperbox is the Cartesian product of $\ttb[k] \coloneqq \bigl[v_{\phi(m-(n-t)+1)}[k], v_{\phi(n-t)}[k]\bigr]$ for all $k\in[d]$.
\end{definition}

In a similar manner, it is proved in \cite{centroid-paper} that, in order to satisfy the convex validity condition, an algorithm must agree inside the following area: 

\begin{definition}[Safe area~\cite{multidim-approx-agreement}]
    Consider $n$ vectors $\{v_1, \dots ,v_n\}\eqqcolon V$, $t<n/(\max\{3,d+1\})$ of which can be Byzantine. Let $C_1,\ldots, C_{\binom{n}{n-t}}$ be the convex hulls of every subset of $V$ of size $n-t$. The \textit{safe area} is the intersection of these convex hulls: $\bigcap_{i\in \left[\binom{n}{n-t}\right]} C_i.$
\end{definition}

\section{Centroid approximation in Byzantine federated learning}\label{sec: theory}

In this section, we first consider approximation guarantees that are given by validity conditions only. We show that only the box validity condition guarantees a bounded approximation ratio of the $\trueBar$. In the second part, we consider the best possible approximation that can be achieved under various validity conditions. We provide tight approximation bounds for each validity condition, apart from the box validity condition, where a gap remains for some specific values of $n$ and $d$. We conclude this section with a discussion on how our results can be transferred to federated learning in a peer-to-peer network. 

\subsection{Approximation guarantees given by validity conditions}\label{sec:validity_bad}

In this section, we show that weak, strong, and convex validity conditions are not sufficient to guarantee that an algorithm achieves a bounded approximation ratio of $\trueBar$.

\begin{lemma}\label{lemma:only_weak_validity}
    Satisfying weak validity is not a sufficient condition for an algorithm to achieve a bounded approximation ratio of $\trueBar$. 
\end{lemma}
\begin{proof}
Without loss of generality, we can consider an algorithm that either agrees on the unique input vector, or outputs the origin. 
Now consider the case where all clients have input $x\cdot (1, \dots, 1)$. Then, the diameter of the minimum covering  ball can be arbitrarily small, but the distance between the origin and $x\cdot (1, \dots, 1)$ is $\sqrt{d}\cdot x$. Hence, the ratio between this distance and the radius of the minimum covering  ball is unbounded. 

\end{proof}

\begin{lemma}\label{lemma:only_strong_validity}
    Satisfying strong validity is not a sufficient condition for an algorithm to achieve a bounded approximation ratio of $\trueBar$. 
\end{lemma}
\begin{proof}
    As before, we can consider an algorithm that either agrees on the unique non-faulty input vector, or outputs the origin (we do not need to know how the algorithm achieves this, only that it is a general algorithm satisfying strong validity). 
    Assume the case, where the $n-t$ non-faulty input vectors are all $\epsilon$ away from $(1,\dots, 1)$, and the Byzantine clients do not send any vector. The distance between the origin and the average of the non-faulty vectors is $\sqrt{d}\cdot x$. The radius of the minimum covering  ball is however $0$. Hence, the approximation ratio is unbounded. 
\end{proof}

\begin{lemma}[from (\cite{centroid-paper}, Observation 4.1)]\label{lemma:only_convex_validity}
    The worst-case approximation ratio that can be achieved by any algorithm satisfying convex validity is unbounded.
\end{lemma}

Next, we show that the box validity condition is the only validity condition that, by itself, guarantees that any algorithm satisfying it has a bounded approximation ratio. 
More precisely, we show that outputting a vector inside $\tb$ is sufficient to ensure that the output is a bounded approximation of $\trueBar$. 

\begin{lemma}\label{lem:box_validity_approx}
    The worst-case approximation ratio that can be achieved by any algorithm satisfying box validity is at most $\frac{t}{n-t}\cdot 2\sqrt{d}$.
\end{lemma}
\begin{proof}
  Consider the coordinate $k\in[d]$ in which $\ttb$ realizes its longest edge. We define a bijection $\phi:[n]\rightarrow[n]$ such that, $i<j \Rightarrow v_{\phi(i)}[k] < v_{\phi(j)}[k], \forall i, j\in[n]$.
  Then, 
  \begin{align*}
    |\bb [k]| &= \frac{1}{n-t}\sum_{i=t+1}^n v_{\phi(i)}- \frac{1}{n-t}\sum_{i=1}^{n-t}v_i[k]\\
    &= \frac{1}{n-t}\sum_{i=n-t+1}^n v_{\phi(i)} +\frac{1}{n-t}\sum_{i=t+1}^{n-t} v_{\phi(i)} - \frac{1}{n-t}\sum_{i=1}^{t}v_i[k] - \frac{1}{n-t}\sum_{i=t+1}^{n-t} v_{\phi(i)} \\
    &= \frac{1}{n-t}\sum_{i=n-t+1}^n v_{\phi(i)} - \frac{1}{n-t}\sum_{i=1}^{t}v_i[k]
    \geq \frac{t}{n-t}v_{\phi(n-t)}- \frac{t}{n-t}v_{\phi(t)} = \frac{t}{n-t}|\ttb[k]|.
  \end{align*}
    
    Since $\bb$ and $\ttb$ are necessarily intersecting~\cite{centroid-paper}, the furthest a vector satisfying box validity can be from $\trueBar$ is if $\trueBar$ is in $\bb$ and the vector is on the opposite vertex of $\ttb$. 
    We showed above that the diagonal of $\ttb$ is at most $\frac{t}{n-t}$ times the diagonal of $\bb$. 
  
    The diagonal of $\bb$ being upper bounded by $2\sqrt{d}\cdot \radiusEncBall$, the furthest we can be from $\trueBar$ by satisfying box validity is 
    \begin{align*}
        \left(1+\frac{t}{n-t}\right)\cdot 2\sqrt{d}\cdot \radiusEncBall.   
    \end{align*}

    The centroid approximation ratio of any algorithm satisfying box validity will hence be upper bounded by $\left(1+\frac{t}{n-t}\right)\cdot 2\sqrt{d}$. 
    
\end{proof}

\subsection{Upper and lower bounds for centroid approximation}\label{sec:traditional_bounds}

In this section, we present upper and lower bounds for centroid approximation under different validity conditions. 
An overview of these results is presented in Table~\ref{tab:overvie_cent_sec}. Note that most bounds are tight. Only in the case $n<d$, there is a gap for approximation under box validity that remains to be investigated. 

In the following, we present the upper bound for weak validity.

\begin{lemma}[upper bound for weak validity]\label{lem:ub_weak_validity}
    The best approximation ratio that can be achieved by an algorithm satisfying weak validity is $1$ in the worst case.
\end{lemma}
\begin{proof}
    We can achieve $1$ with the optimum algorithm picking the center of the minimum covering ball (see~\cite{centroid-paper}). This algorithm satisfies weak validity. 
\end{proof}

Note that this upper bound is tight, as the lower bound cannot be less than $1$ by definition. 
We now present the algorithm that highlights the upper bound for strong validity.

\begin{lemma}[Upper bound for strong validity]\label{lem:MDA}
    The MDA algorithm~\cite{jungle} outputs the average of the subset of $n-t$ vectors that have the smallest diameter, this diameter is defined as the maximum distance between any two vectors. The MDA computes a $2$-approximation of the centroid.
\end{lemma}
\begin{proof}
    Observe that the output vector of the MDA algorithm is in $\barSet$ and is thus inside $\encBall$. The largest distance between any two vectors in $\encBall$ is upper bounded by the diameter of the ball. Thus, the algorithm computes at most a $2$-approximation.
\end{proof}

The following lemma gives a lower bound of $2$ on the approximation ratio of the centroid in the context of strong validity, which matches the upper bound above. This shows that the approximation ratio of the MDA algorithm is tight.

\begin{lemma}[Lower bound for strong validity \cite{centroid-paper}]\label{lem:lb_strong_validity}
    The best approximation ratio that can be achieved by an algorithm satisfying strong validity is $2$ in the worst case.
\end{lemma}

In \cite{centroid-paper}, a lower bound of $2d$ has been presented for convex validity for the worst case where $n=(d+1)t$. We generalize this bound to hold for any $n\ge(d+1)t$.

\begin{lemma}[Lower bound for convex validity]\label{lem:lb_convex_validity}
    The best approximation ratio that can be achieved by an algorithm satisfying convex validity is at least $2d$.
\end{lemma}
\begin{proof}
    In \cite{centroid-paper}, a lower bound of $2d$ has been shown for the worst case $n=(d+1)t$. This proof can be easily extended to hold for the general case $n>\max\{3,d+1\}\cdot t$. Assume that $dt$ vectors are placed at coordinates $x + \varepsilon\cdot u_i, i\in \{1,\ldots, d\}$, where $\varepsilon$ is a small constant and $t$ vectors placed at each coordinate. The remaining $n-dt$ vectors are placed at $(0,\ldots,0)$. Assume that these $n-dt$ vectors include $t$ Byzantine vectors. Observe that such a construction is always possible since $n>(d+1)t$.
    
    In \cite{centroid-paper}, it was shown that the safe area of such a construction results in a single point $(0,\ldots,0)$. Note that the non-faulty centroid is located in $td/(n-t)$, and the radius of the centroid ball is $t/(2(n-t))$. Thus, the approximation of the centroid is $2d$ in this example.
\end{proof}

\begin{table}[h]
    \centering
    \begin{tabular}{l|c|c|c}
         \makecell{{validity} {condition}} & \makecell{{LB} for $n>(d+1)t$} & \makecell{{LB} for $n<(d+1)t$} & \makecell{{upper} {bound}} \\ 
         \Xhline{2\arrayrulewidth}
         weak  & $1$ & $1$ & $1$ \tiny{(Lemma~\ref{lem:ub_weak_validity})} \\ 
         strong  & $2$ \tiny{\cite{centroid-paper}} & $2$ \tiny{\cite{centroid-paper}} & $2$ \tiny{(Lemma~\ref{lem:MDA})} \\ 
         box  & $\sqrt{d}$ \tiny{(Lemma~\ref{lem:lb_box_validity})} & $\min\{\frac{n-t}{t},\sqrt{d}\}$ \tiny{(Lemma~\ref{lem:lb_box_validity})} & $2\sqrt{\min\{n,d\}}$ \tiny{(Lemma~\ref{lem:box_UB})} \\ 
         convex  & $2d$ \tiny{(\cite{centroid-paper}}, Lemma~\ref{lem:lb_convex_validity}) & not possible~\tiny{\cite{multidim-approx-agreement}} & $2d$ \tiny{(Lemma~\ref{lem:ub_convex_validity})} \\
    \end{tabular}
    \caption{This is an overview of the results established in this section. Already known results are cited in the respective cells. The lower bound on weak validity follows from the definition of approximation.}
    \label{tab:overvie_cent_sec}
\end{table}

We next give an upper bound result for the box validity condition. Note that there are two algorithms in the literature that achieve the same approximation ratio.

\begin{lemma}[Upper bound for box validity]\label{lem:box_UB}
    One round of the Box algorithm~\cite{centroid-paper} or the RB-TM algorithm~\cite{jungle} achieves an approximation ratio of $2\sqrt{\min\{n,d\}}$.
\end{lemma}
\begin{proof}
Note that both algorithms were presented to solve approximate agreement. We can however let the server run one round of these algorithms as if the server were one of the nodes in the distributed network. In~$\cite{centroid-paper}$, it was shown that the output vector of one node at the end of a round is inside the intersection of $\bb$ and $\ttb$. 
This condition is sufficient to achieve a $2\sqrt{d}$-approximation~$\cite{centroid-paper}$. This solves the case $n>d$.

We next consider the case $n<d$. 
Note that if $\bb$ has dimension $n$, the diagonal length argument from \cite{centroid-paper} implies a $2\sqrt{n}$ bound on the approximation ratio. 
Suppose that $\bb$ has dimension $d'$ where $n<d'\leq d$. 
Since there are $n$ input vectors and all elements of $\barSet$ are computed from those vectors, $\convexHull(\barSet)$ has to be contained in a subspace $U_{\mathrm{input}}$ of dimension $n$.
The hyperbox $\bb$ of dimension $d'$ is the smallest possible hyperbox containing the convex polytope $\convexHull(\barSet)$. 
Hence, $\convexHull(\barSet)$ has to intersect all $2d'$ faces of $\bb$, otherwise there exists a hyperbox strictly contained in $\bb$ that contains $\convexHull(\barSet)$. 
For the sake of simplicity, assume that $\bb$ is the unit hypercube of dimension $d'$ placed at the origin with non-negative coordinates only. Note that translation and rotation of all points do not influence the approximation ratio. Further, all following computations can be adjusted with the length of the longest edge of $\bb$ to achieve the same result in the general case. 

Observe that $\convexHull(\barSet)$ has to intersect all faces of $\bb$ that contain the origin. 
Consider the set of centroids in $\barSet$ that lie on these $d'$ faces. 
Any two such centroids that lie on different faces are linearly independent.
Since $\convexHull(\barSet)$ spans at most an $n$-dimensional subspace, at most $n$ centroids in this set can be linearly independent. 
Note that the radius of $\encBall$ is maximized when the centroids lie on intersections of many faces. 
Consider the largest subset of linearly independent centroids that intersect the $d'$ considered faces (this subset can be chosen greedily). On average, each centroid in this subset lies in the intersection of at least $d'/n$ faces. 
Thus, at least one of these centroids must lie in the intersection of at least $d'/n$ faces of the unit hypercube.
This implies that the radius of the minimum covering ball is at least $\sqrt{d'/n}/2$ (the intersection of $k$ faces is at distance $\sqrt{k}/2$ from the center of the hyperbox).

However, since the centroid of non-faulty vectors has to be contained inside $\convexHull(\barSet)\subseteq \bb$, the distance between the output of an algorithm agreeing inside $\bb$ and $\trueBar$ centroid is at most $\sqrt{d'}$, hence the approximation ratio is at most $2\cdot \sqrt{n}$.

Hence, the approximation ratio of the hyperbox algorithm is at most $2\cdot \sqrt{\min\{n,d\}}$.\qedhere

\end{proof}

Before addressing the lower bound for algorithms satisfying box validity, we first need to prove the following lemma. 

\begin{lemma}\label{lem: box-val-implies-trimmed-box-agreement}
    An algorithm satisfying box validity has to agree inside the trimmed trusted hyperbox.
\end{lemma}

\begin{proof}
    We assume that $t$ Byzantine parties follow the algorithm with their own (worst-case) input vectors, thus being undetectable. 
    Let us consider a consensus algorithm such that the output vector $v$ always satisfies box validity. For the sake of contradiction, suppose this output vector is outside the trimmed trusted hyperbox. 
    By definition of the trimmed trusted hyperbox, there exists a coordinate $k$ for which $v[k]$ is strictly larger than $n-t$ of the input vectors at coordinate $k$. 
    Since Byzantine clients are undetectable, these $n-t$ input vectors could be the non-faulty ones. This implies that the output vector $v$ is not in the trusted box, thus violating the box validity condition. This is a contradiction. 
    Hence, the output vector of any algorithm satisfying the box validity condition must be in the trimmed trusted hyperbox. 
\end{proof}

\begin{lemma}[Lower bound for box validity]\label{lem:lb_box_validity}
    The approximation ratio of any algorithm satisfying box validity is at least $\sqrt{\frac{1}{2}\cdot \min\left\{\left\lfloor\frac{n-t}{t}\right\rfloor,d\right\}}$ .
\end{lemma}

\begin{proof}
    In order to prove the lower bound on the approximation ratio, we present a construction where the trimmed trusted hyperbox consists of just one vector. 
    Consider a setting where $n-t-\min\left\{\lfloor\frac{n-t}{t}\rfloor t,dt\right\}$ input vectors are at coordinate $(0,\ldots,0)$. We further assume that $t$ vectors are at coordinate $e_k = x\cdot u_k, \forall k\in \left[\min\left\{\lfloor\frac{n-t}{t}\rfloor,d\right\}\right]$, where $u_k$ is the $k^{th}$ unit vector. Suppose the $t$ Byzantine vectors choose their input vectors to be $(0,\ldots,0)$. Then, the trimmed trusted hyperbox is $(0, \dots, 0)$.

    The centroid of non-faulty vectors is 
        $\frac{t}{n-t}\sum_{k=1}^{\min\{\lfloor(n-t)/t\rfloor,d\}}e_k$
    and the distance between the trimmed trusted hyperbox and $\trueBar$ is 
    \begin{align*}
        \distance\big(\trueBar,(0,\ldots,0)\big) &= \sqrt{\sum\nolimits_{k=1}^{\min\{\lfloor(n-t)/t\rfloor,d\}} \left(\frac{t}{n-t}\cdot x\right)^2} \\
        &= \sqrt{\min\left\{\left\lfloor\frac{n-t}{t}\right\rfloor,d\right\}\cdot \left(\frac{t}{n-t}\cdot x\right)^2}
        = \sqrt{\min\left\{\left\lfloor\frac{n-t}{t}\right\rfloor,d\right\}}\cdot \frac{tx}{n-t}.
    \end{align*}

   Now the radius of the minimum covering  ball is at most the largest distance between two possible centroids:

    \begin{align*}
        \radiusEncBall&\leq \left\Vert \sum_{k=2}^{\min\{\lfloor(n-t)/t\rfloor,d\}}\left(\frac{t}{n-t}\cdot e_k\right) - \hspace{-4pt}\sum_{k = 1}^{\min\{\lfloor(n-t)/t\rfloor,d\}-1} \left(\frac{t}{n-t}\cdot e_k\right)\right\Vert_2 
        = 2\cdot\sqrt{\left(\frac{tx}{n-t}\right)^2}.
    \end{align*}

    Hence, the approximation ratio is at least
    \begin{align*}
        \frac{\distance\big(\trueBar,(0,\ldots,0)\big)}{\radiusEncBall}
        \ge \sqrt{\frac{1}{2}\cdot \min\left\{\left\lfloor\frac{n-t}{t}\right\rfloor,d\right\}}
    \end{align*}

\end{proof}

We finally consider convex validity. Note that no guarantees can be given for algorithms satisfying convex validity in the case $n>\max\{3,d+1\}$ since the safe area cannot be guaranteed to exist in such cases. The results presented here are therefore only of interest in applications where the number of clients surpasses the dimension of the training model.

\begin{lemma}[Upper bound for convex validity]\label{lem:ub_convex_validity}
    Consider the algorithm that outputs a vector contained in the \textit{safe area} that minimizes the distance to the center of $\encBall$. This algorithm computes a $2d$-approximation of the centroid.
\end{lemma}
\begin{proof}
    Observe that the algorithm computes at most a $2$-approximation of $\trueBar$ if the \textit{safe area} and $\encBall$ intersect. This is because the algorithm then chooses a vector that is contained in $\encBall$.

    Now we consider the remaining case, where the \textit{safe area} and $\encBall$ are disjoint. Let $x$ denote the distance between the \textit{safe area} and $\encBall$ and let $S$ denote the closest point of the \textit{safe area} to $\encBall$ and $B$ the closest point of $\encBall$ to the \textit{safe area}, so that the distance between $S$ and $B$ is $x$. We start by projecting all input vectors orthogonally onto $\overline{S,B}$. The approximation ratio of the algorithm is computed as $(x+\radiusEncBall)/\radiusEncBall$. Observe that the distance between any two centroids after their orthogonal projection onto $\overline{S,B}$ cannot increase due to the triangle inequality, while the distance between $S$ and $B$ remains unchanged. Therefore, the distance between any two projected centroids onto $\overline{S,B}$ is a lower bound on the diameter of $\encBall$. To simplify the discussion on distances, we assume that $S$ is at coordinate $0$ and $B$ is at coordinate $x$. 

    In the following, we will lower bound the size of $\radiusEncBall$ and upper bound the size of $x$. Let the projection of the vectors $v_1,\ldots, v_n$ be denoted $p_1,\ldots, p_n$ such that the vector projected on the smallest coordinate is denoted $p_1$ and the one projected on the largest coordinate is $p_n$. 
    Note that there must be at least $t+1$ vectors $p_i$ having negative coordinates, otherwise there would exist a convex hull of $n-t$ vectors that would project onto only strictly positive coordinates, which is a contradiction. 
    There are also at least $t+1$ vectors $p_j$ that have coordinate at least $x$. If this was not true, there would exists a centroid with a smaller coordinate than $x$, which is a contradiction. Further, there are at most $td$ vectors with a positive coordinate (see proof of Lemma~\ref{lem:safe_area_left_of_hyperplane}).

    Let $l$ denote the number of vectors $p_i$ with negative coordinates. Let $r$ denote the number of vectors $p_i$ with a larger coordinate than $x$, and let $y_1, \ldots, y_r$ denote the coordinates of these vectors in increasing order. Further, we say that the smallest $r-t$ coordinates have an average value of $\overline{y}_{min}$ while the largest $t$ coordinates have an average of $\overline{y}_{max}$. The average of all vectors $p_i$ with coordinates between $0$ and $x$ is defined to be $a$.

    Observe that $x$ is upper bounded by the coordinate of any possible centroid. We choose the following centroid to upper bound $x$: the average of some $t+1$ vectors with negative coordinates, all the vectors between $0$ and $x$, and the remaining smallest $r-t$ vectors with coordinates larger than $x$. This gives the following bound:

    \begin{align*}
        x &\le \frac{1}{n-t}\left(\sum\nolimits_{i=1}^{r-t} y_i + a\cdot(n-r-l)\right) \le  \frac{1}{n-t}(n-t-l)\cdot\overline{y}_{min}
    \end{align*}
    note that we upper bounded all vectors with coordinates smaller than $0$ by $0$.

    To lower bound the diameter of $\encBall$, we consider the difference between its largest and smallest coordinates:

    \begin{align*}
        \radiusEncBall \ge \frac{1}{2(n-t)}\left(\sum_{i=t+1}^{n} p_i - \sum_{i=1}^{n-t} p_i \right)
        \ge \frac{1}{2(n-t)}\left( t\cdot\overline{y}_{max} - \sum_{i=1}^{t} p_i\right) \ge \frac{t}{2(n-t)}\cdot\overline{y}_{max}
    \end{align*}
    where $\frac{1}{n-t}\sum_{i=1}^{t} p_i \le 0$ since there are at least $t+1$ vectors $p_i$ with negative coordinates.
    
    The approximation ratio achieved by the algorithm can now be upper bounded by: 

    \begin{align*}
        \frac{x}{\radiusEncBall} +1 \le \frac{\frac{1}{n-t}(n-t-l)\cdot\overline{y}_{min}}{\frac{t}{2(n-t)}\cdot\overline{y}_{max}} +1
        \le \frac{2(n-t-l)}{t} +1 \le \frac{2dt}{t} +1= 2d +1.
    \end{align*}

    The last inequality holds because there can be at most $dt$ vectors with positive coordinates, i.e., $n-t-l \le dt$. \qedhere

\end{proof}

\begin{lemma}\label{lem:safe_area_left_of_hyperplane}
    Assume that the \textit{safe area} is a $q$-dimensional convex polytope, where $1\le q\le d$. Consider the $q$-dimensional subspace in which the \textit{safe area} is defined. Let $H$ be a hyperplane that touches the \textit{safe area} and divides the $q$-dimensional space into two subspaces. Then, there can be at most $qt$ points on the opposite side of $H$ wrt. the \textit{safe area}. 
\end{lemma}
\begin{proof}
    
    Consider a vertex $s_v$ of the safe area that lies at the intersection of the \textit{safe area} with the hyperplane $H$. Note that at least one such vertex must exist since the \textit{safe area} is a convex polytope. 
    
    Observe that exactly $q$ $(q-1)$-faces of \textit{safe area} meet in $s_v$. Each of these faces are  hyperplanes , denoted $H_1,\ldots, H_d$, and go through $s_d$, each of them defined by a face of the safe area. The \textit{safe area} is defined such that, for each face $F_i$, at most $t$ vectors can lie outside of \textit{safe area} and thus on the opposite side of $H$ w.r.t. \textit{safe area}. In total, at most $qt$ can lie on the opposite side of $H$. And at least $n-qt>n-dt$ vectors must lie inside \textit{safe area}. 
\end{proof}

\subsection{Federated learning in peer-to-peer networks}

The results presented in this paper also hold for federated learning in synchronous peer-to-peer networks. In the peer-to-peer setting, there is no trusted server. Instead, the clients communicate with each other in a fully-connected network by sending messages. The aggregation step by the server is replaced by an exact Byzantine agreement algorithm that makes sure that the clients agree on the same aggregation vector.

The lower bounds presented in Section~\ref{sec:validity_bad} and~\ref{sec:traditional_bounds} trivially extend to this distributed setting, as they are presented for a stronger setting in which the clients do not receive different sets of vectors as it is possible in a peer-to-peer setting. On the other hand, interactive consistency protocols~\cite{IC1,IC2} from distributed computing allow the clients to agree on the same set of vectors. Thus, each client can apply the aggregation algorithms presented in this paper locally. Since these algorithms are deterministic, all clients would output identical vectors after Byzantine agreement.

\section{Empirical evaluation}\label{sec:practical}

In the practical evaluation, we differentiate between the two FL variants where the model parameters or the gradients are exchanged. We consider $n$ clients, where each client $i\in [n]$ has access to its own data that follows an unknown distribution $\mathcal{D}_i$. Let $F_i(x)$ be the local loss function of client $i$ with respect to model parameter $x$. The objective is $$\argmin\nolimits_{x\in\mathbb{R}^d} F(x),\quad \text{where}\quad F(x)=\frac{1}{n}\sum\nolimits_{i=1}^{n}F_i(x)$$
We do not differentiate between the amounts of data points that a client holds, as a Byzantine client could lie about its data to have a larger influence. The training is executed in rounds. We differentiate between the following two settings:

\paragraph*{\fedsgd}
In each round $r$, a client locally computes the gradient $g_i(x_r) = \nabla F_i(x_r)$ on its dataset. It then sends $g_i(x_r)$ to the server. The server upgrades the global model by aggregating the gradients $x_{r+1} \leftarrow x_r - \eta\frac{1}{n} \sum_{i=1}^{n} g_i(x_r)$, where $\eta$ is a fixed learning rate, and sends the new model to the clients for the next round.

\paragraph*{\fedavg}
In each round $r$, a client locally updates its model parameters (possibly multiple times) $x_{r+1}^i \leftarrow x_{r}^i - \eta g_i(x_r)$. It then shares its model parameter $x_{r+1}^i$ with the server. The server aggregates the model parameters $x_{r+1} \leftarrow \sum_{i=1}^{n} x_{r+1}^i$ and shares the new model with the clients.

\medskip The aggregation algorithm in the definition of \fedsgd and \fedavg is an unweighted average of the vectors. For the experiments, we replace this aggregation step with one of the aggregation algorithms presented in Section~\ref{sec:traditional_bounds}. These aggregation algorithms are summarized below.

\paragraph*{Aggregation algorithms}
We implemented the following aggregation algorithms for comparison: 
\begin{itemize}
    \item \textbf{Center of $\encBall$:} This algorithm computes all possible centroids on subsets of $n-t$ vectors and outputs the center of $\encBall$. The algorithm achieves a $1$-approximation of the centroid and satisfies weak validity.
    \item \textbf{MDA\cite{jungle}:} This algorithm computes a subset of $n-t$ vectors with the smallest diameter and outputs the centroid of this subset. The algorithm achieves a $2$-approximation of the centroid and satisfies strong validity. 
    \item \textbf{Box Algorithm\cite{centroid-paper}:} This algorithm computes the intersection of $\ttb$ and $\bb$, and outputs the center of this intersection. In \cite{centroid-paper}, it was shown that such an intersection is non-empty for $n>3t$. The algorithm achieves a $2\sqrt{d}$-approximation of the centroid and satisfies box validity.
\end{itemize}

We do not implement the algorithm based on the \textit{safe area} (see Lemma~\ref{lem:ub_convex_validity}), since this algorithm only works in scenarios where $n>(d+1)t$. Considering that our training models have 200 dimensions, such an experiment was not feasible in our experimental setup.

\subsection{Experimental setup}
We implement a client/server federated learning model for solving classification tasks in Python using the Tensorflow library. 
The models are evaluated on the MNIST dataset from Kaggle\footnote{\url{https://www.kaggle.com/datasets/scolianni/mnistasjpg}, accessed on 13.05.2025}. The dataset contains 42,000 images of handwritten digit in JPEG format which are labeled, and each class of the data is kept in a separate folder. 
We consider 10 clients with 3 different data distributions: homogeneous, mild heterogeneous and extreme heterogeneous. For the homogeneous data distribution, we shuffle the entire dataset and split it among 10 clients. The mild heterogeneity case splits each class into $10$ parts, where $8$ parts contain $10\%$ of the class, one part $5\%$ and one part $15\%$ of the class. In the extreme heterogeneity case, the data is sorted and split into 20 partitions. Each client is randomly assigned two partitions, ensuring they contain data from two distinct classes.
Note that the scenarios where clients have different local dataset sizes are not taken into account, as Byzantine clients could exploit this variation to their advantage.  

The underlying neural network for solving the image classification task is a MultiLayer Perceptron (MLP) with 3 layers. The learning rate is set to $\eta = 0.01$ and the decay is calculated with respect to the number of global communication rounds (epochs), i.e. $decay = \frac{\eta}{rounds}$.

Byzantine behavior in federated learning has been extensively studied and the attacks are categorized into training-based and parameter-based attacks \cite{shi2022challenges}. While training-based attacks, also known as data poisoning attacks, have been analyzed in \cite{biggio2012poisoning, mahloujifar2019data, farhadkhani2024brief}, our work focuses on parameter-based attacks \cite{shi2022challenges, farhadkhani2022byzantine}. Specifically, we implement the sign flip attack, inspired by the signSGD algorithm \cite{jin2020signSGD, bernstein2018signsgd}, where gradients are replaced by their sign values during transmission. In the sign flip attack, the gradient of the faulty clients is multiplied by $-1$ and sent to the server.
This attack has been widely used in practical simulations \cite{variance-reduced-sgd, elite, farhadkhani2022byzantine, gradient-filtering2022, probabilistic-signflip}.

\subsection{Experimental results}
\begin{figure}[h]
    \centering
    \begin{subfigure}{0.48\textwidth}
        \centering        \includegraphics[width=\textwidth]{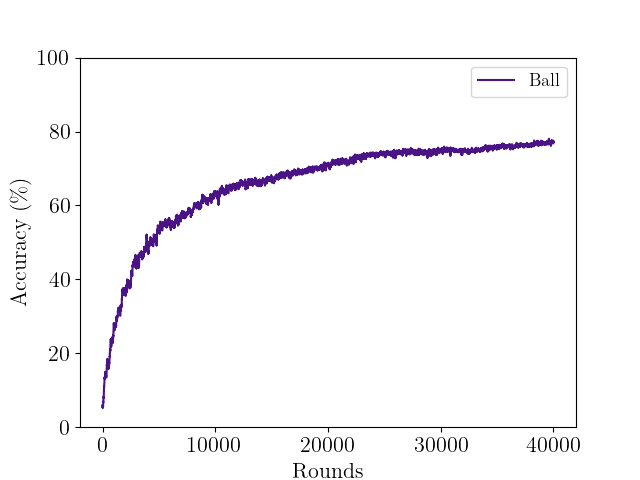}
        \caption{Center of $\encBall$ without failures.}
        \label{fig:ball}
    \end{subfigure}
    \hfill
    \begin{subfigure}{0.48\textwidth} 
        \centering        \includegraphics[width=\textwidth]{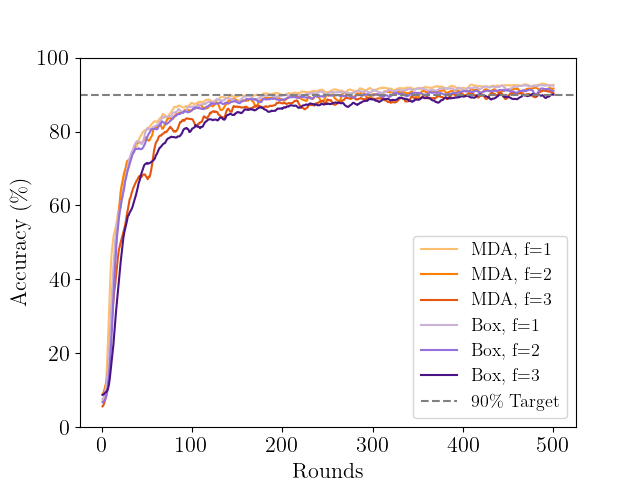} 
        \caption{MDA and Box with different numbers of failures.}
        \label{fig:fedsgd-homo}
    \end{subfigure}   
    \caption{$\fedsgd$ setting on homogeneous data with MDA, Box and $\encBall$ algorithm}
    \label{fig:figure2}
\end{figure}

\begin{figure}[H]
    \centering
    \begin{subfigure}{0.48\textwidth}
        \centering        \includegraphics[width=\textwidth]{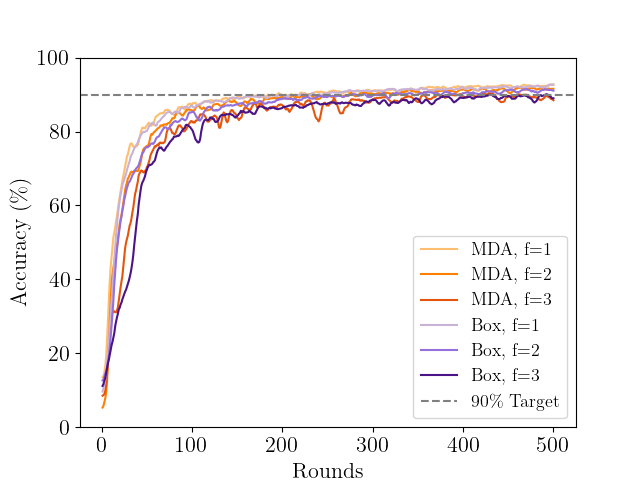}
        \caption{$\fedsgd$ with mild heterogeneous data.}
        \label{fig:sub2}
    \end{subfigure}
    \hfill
    \begin{subfigure}{0.48\textwidth}
        \centering        \includegraphics[width=\textwidth]{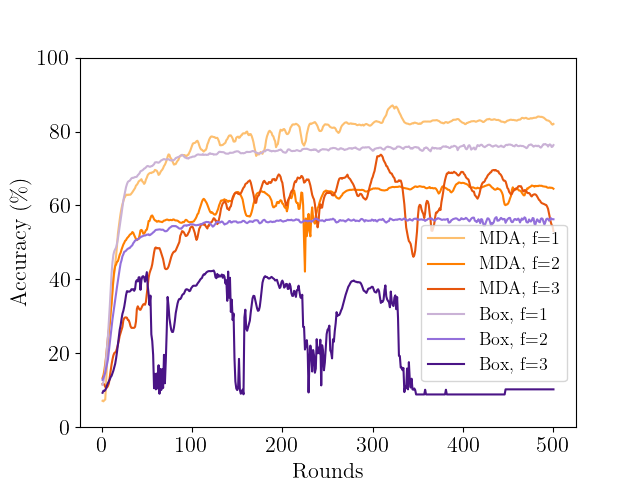}
        \caption{$\fedsgd$ with extreme heterogeneous data.}
        \label{fig:sub3}
    \end{subfigure}
    \begin{subfigure}{0.48\textwidth}
        \centering        \includegraphics[width=\textwidth]{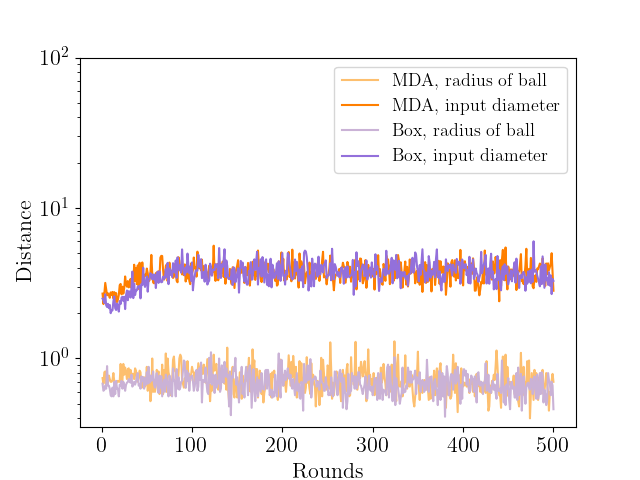}
        \caption{Radius of the ball and diameter of non-faulty vectors with mild heterogeneous data distribution in $\fedsgd$ setting with $f=1$.}
        \label{fig:rad-mild-hetero}
    \end{subfigure}
    \hspace{10pt}
    \begin{subfigure}{0.48\textwidth}
        \centering        \includegraphics[width=\textwidth]{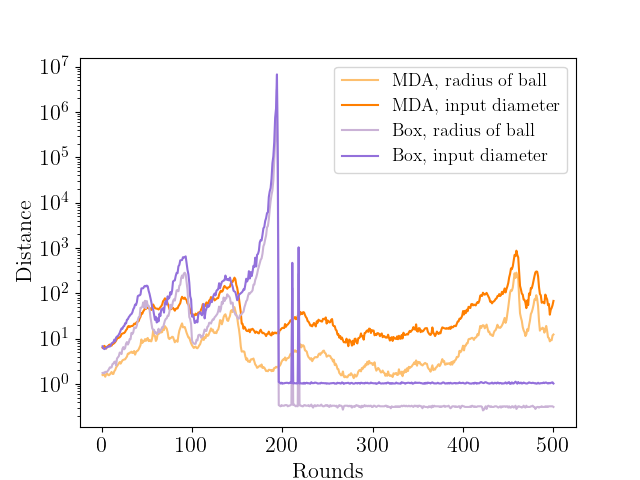}
        \caption{Radius of the ball and diameter of non-faulty vectors with extreme heterogeneous data distribution in $\fedsgd$ setting with $f=3$.}
        \label{fig:rad-extreme}
    \end{subfigure}
    \caption{$\fedsgd$ with mild heterogeneous and extreme heterogeneous data under sign flip attack.}
    \label{fig:fedsgd}
\end{figure}

\begin{figure}[h]
    \centering
    \begin{subfigure}{0.48\textwidth}
        \centering        \includegraphics[width=\textwidth]{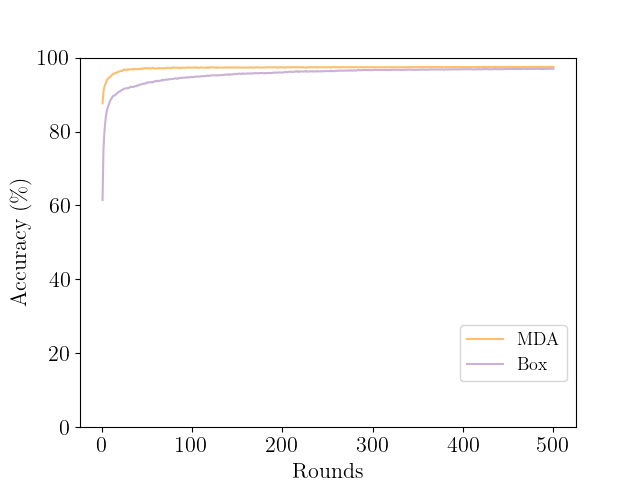}
        \caption{$\fedavg$ without attacks}
        \label{fig:fedavg_noattack}
    \end{subfigure}
    \hfill  
    \begin{subfigure}{0.48\textwidth}
        \centering        \includegraphics[width=\textwidth]{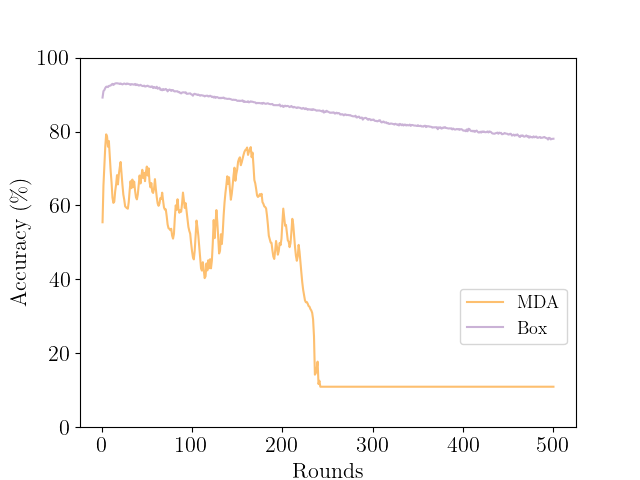}
        \caption{$\fedavg$ under the sign flip attack $f=1$}
        \label{fig:fedavg-homog}
    \end{subfigure}
    \caption{$\fedavg$ setting with homogeneous data distributions with and without Byzantine behavior}
    \label{fig:fedavg}
\end{figure}

Figure \ref{fig:ball} illustrates the Center of $\encBall$ algorithm in the $\fedsgd$ setting with no Byzantine behavior. It can be observed that after 40,000 rounds $\encBall$ algorithm reaches over $77\%$. The Center of $\encBall$ algorithm requires significantly more rounds than the MDA or the Box algorithm in Figure \ref{fig:fedsgd-homo}
and is therefore not evaluated under Byzantine behavior and different heterogeneity distributions. 

In Figure \ref{fig:fedsgd-homo} we evaluate MDA and Box algorithm in the $\fedsgd$ setting with homogeneous data distribution under the sign flip attack. We set the number of Byzantine clients to be $f=1$, $f=2$ or $f=3$. MDA and the Box algorithm converge and achieve over 90$\%$ accuracy, even with $f>1$.

In Figure~\ref{fig:sub2}, the mild heterogeneous case  shows a slight decrease of $2\%$ in accuracy with $f>1$, compared to the homogeneous case. 
In Figure \ref{fig:fedsgd-homo} and Figure \ref{fig:sub2} a drop in accuracy between the different number of adversaries can be noticed. With the mild heterogeneous data in Figure~\ref{fig:sub2} and $f=3$, MDA shows more instability between rounds 220 and 270 with sudden changes in accuracy by up to 7$\%$. Box algorithm appears to be more stable in that case.

Figure~\ref{fig:sub3} illustrates the extreme heterogeneous case, where each client has two classes of data. The Box algorithm with $f=1$ and $f=2$ converges and achieves $76\%$ and $57\%$ accuracy, respectively. However, with $f=3$ the box algorithm fails to converge. MDA with $f=1$ achieves $82\%$ accuracy. 
For $f>1$, MDA becomes unstable and does not converge.
In this case, the MDA algorithm achieves a higher accuracy than the Box algorithm, but it does not seem to converge as robustly as the Box algorithm under one or two Byzantine failures. For $f=3$, the MDA and the Box algorithms both fail to converge.
These results suggest that there may be a trade-off between the centroid approximation and the different validity conditions also in practice, which we plan to investigate in more detail in future work.

Figures~\ref{fig:rad-mild-hetero} depicts radius of the smallest enclosing ball and diameter of the cnon-faulty vectors in $\fedsgd$ setting with mild heterogeneity and $f=1$. A relationship between Figure~\ref{fig:sub2} with $f=1$ and \ref{fig:rad-mild-hetero} can be observed. MDA and Box algorithm converge, therefore diameter of non-faulty input vectors and radius of the ball remain constant throughout the training process. Figure~\ref{fig:rad-extreme} represents the radius of the smallest enclosing ball and diameter of the non-faulty vectors in $\fedsgd$ setting with extreme heterogeneity and $f=3$, which seems to correlate with Figure~\ref{fig:sub3}. It can be noticed that the Box algorithm fails to converge, causing both the radius of the ball and the input diameter to approach zero. 

Figure~\ref{fig:fedavg} shows $\fedavg$ setting with homogeneous data distribution. When there is no Byzantine behavior, as in Figure~\ref{fig:fedavg_noattack}, MDA and Box algorithm converge and achieve over $97\%$ accuracy. However, if there is one Byzantine party in the system, as shown in Figure~\ref{fig:fedavg-homog}, MDA fails to converge. Further, there is a distinct decrease in the accuracy of the Box algorithm, which implies that the sign flip attack has a big impact on the system and prevents the model from converging. In the future, we intend to continue the empirical evaluation and test out various Byzantine attacks in different settings.

\section*{Acknowledgments}
We thank anonymous reviewers for their helpful feedback on previous versions of this work. This work was supported in part by the Research Council of Finland, Grant 333837 and the German Research Foundation (DFG), Schwerpunktprogramm, SPP 2378 (project ReNO), 2023-2027.

\bibliography{references}

\begin{thebibliography}{57}
\providecommand{\natexlab}[1]{#1}
\providecommand{\url}[1]{\texttt{#1}}
\expandafter\ifx\csname urlstyle\endcsname\relax
  \providecommand{\doi}[1]{doi: #1}\else
  \providecommand{\doi}{doi: \begingroup \urlstyle{rm}\Url}\fi

\bibitem[whi(2013)]{white-house-report}
Consumer data privacy in a networked world: A framework for protecting privacy and promoting innovation in the global digital economy.
\newblock \emph{Journal of Privacy and Confidentiality}, 4, 03 2013.
\newblock \doi{10.29012/jpc.v4i2.623}.

\bibitem[Abbas et~al.(2022)Abbas, Shabbir, Li, and Koutsoukos]{abbas2022centerpoint}
W.~Abbas, M.~Shabbir, J.~Li, and X.~Koutsoukos.
\newblock Resilient distributed vector consensus using centerpoint.
\newblock \emph{Automatica}, 136:\penalty0 110046, 2022.
\newblock ISSN 0005-1098.

\bibitem[Alistarh et~al.(2018)Alistarh, Allen-Zhu, and Li]{Alistarh_ByzantineSGD}
D.~Alistarh, Z.~Allen-Zhu, and J.~Li.
\newblock Byzantine stochastic gradient descent.
\newblock In \emph{Advances in Neural Information Processing Systems}, volume~31. Curran Associates, Inc., 2018.

\bibitem[Bar-Noy and Dolev(1988)]{10.1007/BFb0040405}
A.~Bar-Noy and D.~Dolev.
\newblock Families of consensus algorithms.
\newblock In \emph{VLSI Algorithms and Architectures}, 1988.
\newblock ISBN 978-0-387-34770-7.

\bibitem[Bernstein et~al.(2019)Bernstein, Zhao, Azizzadenesheli, and Anandkumar]{bernstein2018signsgd}
J.~Bernstein, J.~Zhao, K.~Azizzadenesheli, and A.~Anandkumar.
\newblock signsgd with majority vote is communication efficient and fault tolerant.
\newblock In \emph{7th International Conference on Learning Representations, {ICLR} 2019, New Orleans, LA, USA, May 6-9, 2019}. OpenReview.net, 2019.

\bibitem[Biggio et~al.(2012)Biggio, Nelson, and Laskov]{biggio2012poisoning}
B.~Biggio, B.~Nelson, and P.~Laskov.
\newblock Poisoning attacks against support vector machines.
\newblock In \emph{Proceedings of the 29th International Coference on International Conference on Machine Learning}, ICML'12, Madison, WI, USA, 2012.

\bibitem[Bracha(1987)]{BrachaRB}
G.~Bracha.
\newblock {Asynchronous Byzantine Agreement Protocols}.
\newblock \emph{{Information and Computation}}, 75\penalty0 (2):\penalty0 130--143, 1987.

\bibitem[Bracha and Toueg(1983)]{10.1145/800221.806706}
G.~Bracha and S.~Toueg.
\newblock Resilient consensus protocols.
\newblock In \emph{Proceedings of the Second Annual ACM Symposium on Principles of Distributed Computing}, PODC '83, 1983.
\newblock \doi{10.1145/800221.806706}.

\bibitem[Cambus and Melnyk(2023)]{centroid-paper}
M.~Cambus and D.~Melnyk.
\newblock Improved solutions for multidimensional approximate agreement via centroid computation, 2023.
\newblock URL \url{https://arxiv.org/abs/2306.12741}.

\bibitem[Chen et~al.(2016)Chen, Monga, Bengio, and Jozefowicz]{45187}
J.~Chen, R.~Monga, S.~Bengio, and R.~Jozefowicz.
\newblock Revisiting distributed synchronous sgd.
\newblock In \emph{International Conference on Learning Representations Workshop Track}, 2016.
\newblock URL \url{https://arxiv.org/abs/1604.00981}.

\bibitem[Civit et~al.(2021)Civit, Gilbert, and Gramoli]{civit2021polygraph}
P.~Civit, S.~Gilbert, and V.~Gramoli.
\newblock Polygraph: Accountable byzantine agreement.
\newblock In \emph{2021 IEEE 41st International Conference on Distributed Computing Systems (ICDCS)}, pages 403--413. IEEE, 2021.

\bibitem[Civit et~al.(2022)Civit, Dzulfikar, Gilbert, Gramoli, Guerraoui, Komatovic, and Vidigueira]{civit2022byzantine}
P.~Civit, M.~A. Dzulfikar, S.~Gilbert, V.~Gramoli, R.~Guerraoui, J.~Komatovic, and M.~Vidigueira.
\newblock Byzantine consensus is $\theta$ (n$^2$): The dolev-reischuk bound is tight even in partial synchrony!
\newblock In \emph{36th International Symposium on Distributed Computing (DISC 2022)}. Schloss Dagstuhl-Leibniz-Zentrum f{\"u}r Informatik, 2022.

\bibitem[Damaskinos et~al.(2018)Damaskinos, El~Mhamdi, Guerraoui, Patra, and Taziki]{pmlr-v80-damaskinos18a}
G.~Damaskinos, E.~M. El~Mhamdi, R.~Guerraoui, R.~Patra, and M.~Taziki.
\newblock Asynchronous {B}yzantine machine learning (the case of {SGD}).
\newblock In \emph{Proceedings of the 35th International Conference on Machine Learning}, volume~80 of \emph{Proceedings of Machine Learning Research}, pages 1145--1154. PMLR, 10--15 Jul 2018.

\bibitem[Data and Diggavi(2021)]{data2021byzantine}
D.~Data and S.~Diggavi.
\newblock Byzantine-resilient high-dimensional sgd with local iterations on heterogeneous data.
\newblock In \emph{International Conference on Machine Learning}, pages 2478--2488. PMLR, 2021.

\bibitem[Dean et~al.(2012)Dean, Corrado, Monga, Chen, Devin, Mao, Ranzato, Senior, Tucker, Yang, Le, and Ng]{dean2012largescale}
J.~Dean, G.~Corrado, R.~Monga, K.~Chen, M.~Devin, M.~Mao, M.~a. Ranzato, A.~Senior, P.~Tucker, K.~Yang, Q.~Le, and A.~Ng.
\newblock Large scale distributed deep networks.
\newblock In \emph{Advances in Neural Information Processing Systems}, volume~25. Curran Associates, Inc., 2012.
\newblock URL \url{https://proceedings.neurips.cc/paper_files/paper/2012/file/6aca97005c68f1206823815f66102863-Paper.pdf}.

\bibitem[Dolev et~al.(1986)Dolev, Lynch, Pinter, Stark, and Weihl]{10.1145/5925.5931}
D.~Dolev, N.~A. Lynch, S.~S. Pinter, E.~W. Stark, and W.~E. Weihl.
\newblock Reaching approximate agreement in the presence of faults.
\newblock \emph{J. ACM}, 33\penalty0 (3):\penalty0 499–516, May 1986.
\newblock \doi{10.1145/5925.5931}.

\bibitem[El~Mhamdi et~al.(2018)El~Mhamdi, Guerraoui, and Rouault]{pmlr-v80-mhamdi18a}
E.~M. El~Mhamdi, R.~Guerraoui, and S.~Rouault.
\newblock The hidden vulnerability of distributed learning in {B}yzantium.
\newblock In \emph{Proceedings of the 35th International Conference on Machine Learning}, volume~80 of \emph{Proceedings of Machine Learning Research}, pages 3521--3530. PMLR, 10--15 Jul 2018.

\bibitem[El-Mhamdi et~al.(2020)El-Mhamdi, Guerraoui, Guirguis, Hoang, and Rouault]{el2020genuinely}
E.-M. El-Mhamdi, R.~Guerraoui, A.~Guirguis, L.~N. Hoang, and S.~Rouault.
\newblock Genuinely distributed byzantine machine learning.
\newblock In \emph{Proceedings of the 39th Symposium on Principles of Distributed Computing}, PODC '20, 2020.

\bibitem[El-Mhamdi et~al.(2021)El-Mhamdi, Farhadkhani, Guerraoui, Guirguis, Hoang, and Rouault]{jungle}
E.~M. El-Mhamdi, S.~Farhadkhani, R.~Guerraoui, A.~Guirguis, L.-N. Hoang, and S.~Rouault.
\newblock Collaborative learning in the jungle (decentralized, byzantine, heterogeneous, asynchronous and nonconvex learning).
\newblock \emph{Advances in neural information processing systems}, 34:\penalty0 25044--25057, 2021.

\bibitem[Fang et~al.(2022)Fang, Yang, and Bajwa]{fang2022bridgebyzantineresilientdecentralizedgradient}
C.~Fang, Z.~Yang, and W.~U. Bajwa.
\newblock Bridge: Byzantine-resilient decentralized gradient descent.
\newblock \emph{IEEE Transactions on Signal and Information Processing over Networks}, 8:\penalty0 610--626, 2022.

\bibitem[Farhadkhani et~al.(2022)Farhadkhani, Guerraoui, Gupta, Pinot, and Stephan]{farhadkhani2022byzantine}
S.~Farhadkhani, R.~Guerraoui, N.~Gupta, R.~Pinot, and J.~Stephan.
\newblock Byzantine machine learning made easy by resilient averaging of momentums.
\newblock In \emph{International Conference on Machine Learning}, pages 6246--6283. PMLR, 2022.

\bibitem[Farhadkhani et~al.(2024)Farhadkhani, Guerraoui, Gupta, and Pinot]{farhadkhani2024brief}
S.~Farhadkhani, R.~Guerraoui, N.~Gupta, and R.~Pinot.
\newblock Brief announcement: A case for byzantine machine learning.
\newblock In \emph{Proceedings of the 43rd ACM Symposium on Principles of Distributed Computing}, PODC '24, pages 131--134, 2024.

\bibitem[Fischer and Lynch(1982)]{IC2}
M.~J. Fischer and N.~A. Lynch.
\newblock A lower bound for the time to assure interactive consistency.
\newblock \emph{Information Processing Letters}, 14\penalty0 (4):\penalty0 183--186, 1982.
\newblock ISSN 0020-0190.

\bibitem[Ghosh et~al.(2019)Ghosh, Hong, Yin, and Ramchandran]{ghosh2019robustfederatedlearningheterogeneous}
A.~Ghosh, J.~Hong, D.~Yin, and K.~Ramchandran.
\newblock Robust federated learning in a heterogeneous environment, 2019.
\newblock URL \url{https://arxiv.org/abs/1906.06629}.

\bibitem[Jee~Cho et~al.(2022)Jee~Cho, Wang, and Joshi]{cho2020clientselectionfederatedlearning}
Y.~Jee~Cho, J.~Wang, and G.~Joshi.
\newblock Towards understanding biased client selection in federated learning.
\newblock In \emph{Proceedings of The 25th International Conference on Artificial Intelligence and Statistics}, volume 151 of \emph{Proceedings of Machine Learning Research}. PMLR, 2022.

\bibitem[Jhunjhunwala et~al.(2022)Jhunjhunwala, Sharma, Nagarkatti, and Joshi]{fedvarp}
D.~Jhunjhunwala, P.~Sharma, A.~Nagarkatti, and G.~Joshi.
\newblock Fedvarp: Tackling the variance due to partial client participation in federated learning.
\newblock In \emph{Uncertainty in Artificial Intelligence}, pages 906--916. PMLR, 2022.

\bibitem[Jhunjhunwala et~al.(2023)Jhunjhunwala, Wang, and Joshi]{fedexp}
D.~Jhunjhunwala, S.~Wang, and G.~Joshi.
\newblock Fedexp: Speeding up federated averaging via extrapolation.
\newblock In \emph{The Eleventh International Conference on Learning Representations}, 2023.

\bibitem[Jin et~al.(2020)Jin, Huang, He, Dai, and Wu]{jin2020signSGD}
R.~Jin, Y.~Huang, X.~He, H.~Dai, and T.~Wu.
\newblock Stochastic-sign sgd for federated learning with theoretical guarantees.
\newblock \emph{arXiv preprint arXiv:2002.10940}, 2020.

\bibitem[Kairouz et~al.(2021)Kairouz, McMahan, Avent, Bellet, Bennis, Bhagoji, Bonawitz, Charles, Cormode, Cummings, et~al.]{kairouz2021advances}
P.~Kairouz, H.~B. McMahan, B.~Avent, A.~Bellet, M.~Bennis, A.~N. Bhagoji, K.~Bonawitz, Z.~Charles, G.~Cormode, R.~Cummings, et~al.
\newblock Advances and open problems in federated learning.
\newblock \emph{Foundations and trends{\textregistered} in machine learning}, 14\penalty0 (1--2):\penalty0 1--210, 2021.

\bibitem[Karimireddy et~al.(2020)Karimireddy, Kale, Mohri, Reddi, Stich, and Suresh]{scaffold}
S.~P. Karimireddy, S.~Kale, M.~Mohri, S.~Reddi, S.~Stich, and A.~T. Suresh.
\newblock {SCAFFOLD}: Stochastic controlled averaging for federated learning.
\newblock In \emph{Proceedings of the 37th International Conference on Machine Learning}, volume 119 of \emph{Proceedings of Machine Learning Research}. PMLR, 2020.

\bibitem[Li et~al.(2019)Li, Xu, Chen, Giannakis, and Ling]{10.1609/aaai.v33i01.33011544}
L.~Li, W.~Xu, T.~Chen, G.~B. Giannakis, and Q.~Ling.
\newblock Rsa: Byzantine-robust stochastic aggregation methods for distributed learning from heterogeneous datasets.
\newblock AAAI'19/IAAI'19/EAAI'19. AAAI Press, 2019.
\newblock ISBN 978-1-57735-809-1.

\bibitem[Li et~al.(2014{\natexlab{a}})Li, Andersen, Park, Smola, Ahmed, Josifovski, Long, Shekita, and Su]{10.5555/2685048.2685095}
M.~Li, D.~G. Andersen, J.~W. Park, A.~J. Smola, A.~Ahmed, V.~Josifovski, J.~Long, E.~J. Shekita, and B.-Y. Su.
\newblock Scaling distributed machine learning with the parameter server.
\newblock In \emph{Proceedings of the 11th USENIX Conference on Operating Systems Design and Implementation}, OSDI'14, page 583–598, USA, 2014{\natexlab{a}}. USENIX Association.
\newblock ISBN 9781931971164.

\bibitem[Li et~al.(2014{\natexlab{b}})Li, Andersen, Smola, and Yu]{NIPS2014_1ff1de77}
M.~Li, D.~G. Andersen, A.~J. Smola, and K.~Yu.
\newblock Communication efficient distributed machine learning with the parameter server.
\newblock In \emph{Advances in Neural Information Processing Systems}, volume~27. Curran Associates, Inc., 2014{\natexlab{b}}.

\bibitem[Li et~al.(2020)Li, Sahu, Zaheer, Sanjabi, Talwalkar, and Smith]{fedprox}
T.~Li, A.~K. Sahu, M.~Zaheer, M.~Sanjabi, A.~Talwalkar, and V.~Smith.
\newblock Federated optimization in heterogeneous networks.
\newblock \emph{Proceedings of Machine learning and systems}, 2:\penalty0 429--450, 2020.

\bibitem[Mahloujifar et~al.(2019)Mahloujifar, Mahmoody, and Mohammed]{mahloujifar2019data}
S.~Mahloujifar, M.~Mahmoody, and A.~Mohammed.
\newblock Data poisoning attacks in multi-party learning.
\newblock In \emph{ICML}, pages 4274--4283, 2019.

\bibitem[McMahan et~al.(2017)McMahan, Moore, Ramage, Hampson, and Arcas]{fedavg}
B.~McMahan, E.~Moore, D.~Ramage, S.~Hampson, and B.~A.~y. Arcas.
\newblock {Communication-Efficient Learning of Deep Networks from Decentralized Data}.
\newblock In \emph{Proceedings of the 20th International Conference on Artificial Intelligence and Statistics}, volume~54 of \emph{Proceedings of Machine Learning Research}, pages 1273--1282. PMLR, 20--22 Apr 2017.
\newblock URL \url{https://proceedings.mlr.press/v54/mcmahan17a.html}.

\bibitem[McMahan et~al.(2016)McMahan, Moore, Ramage, and y~Arcas]{google-federated}
H.~B. McMahan, E.~Moore, D.~Ramage, and B.~A. y~Arcas.
\newblock Federated learning of deep networks using model averaging.
\newblock \emph{arXiv preprint arXiv:1602.05629}, 2\penalty0 (2), 2016.

\bibitem[Melnyk and Wattenhofer(2018)]{box-validity-def}
D.~Melnyk and R.~Wattenhofer.
\newblock Byzantine agreement with interval validity.
\newblock In \emph{2018 IEEE 37th Symposium on Reliable Distributed Systems (SRDS)}, pages 251--260, 2018.
\newblock \doi{10.1109/SRDS.2018.00036}.

\bibitem[Mendes et~al.(2015)Mendes, Herlihy, Vaidya, and Garg]{multidim-approx-agreement}
H.~Mendes, M.~Herlihy, N.~Vaidya, and V.~K. Garg.
\newblock Multidimensional agreement in byzantine systems.
\newblock \emph{Distrib. Comput.}, 28\penalty0 (6), 2015.
\newblock ISSN 0178-2770.

\bibitem[Mitra et~al.(2021)Mitra, Jaafar, Pappas, and Hassani]{fedlin}
A.~Mitra, R.~Jaafar, G.~J. Pappas, and H.~Hassani.
\newblock Linear convergence in federated learning: Tackling client heterogeneity and sparse gradients.
\newblock In \emph{Advances in Neural Information Processing Systems}, 2021.

\bibitem[Pease et~al.(1980)Pease, Shostak, and Lamport]{IC1}
M.~Pease, R.~Shostak, and L.~Lamport.
\newblock Reaching agreement in the presence of faults.
\newblock \emph{J. ACM}, 27\penalty0 (2), Apr. 1980.

\bibitem[Reddi et~al.(2021)Reddi, Charles, Zaheer, Garrett, Rush, Kone{\v{c}}n{\'y}, Kumar, and McMahan]{adaptivefederatedoptimization}
S.~J. Reddi, Z.~Charles, M.~Zaheer, Z.~Garrett, K.~Rush, J.~Kone{\v{c}}n{\'y}, S.~Kumar, and H.~B. McMahan.
\newblock Adaptive federated optimization.
\newblock In \emph{International Conference on Learning Representations}, 2021.

\bibitem[Shamir et~al.(2014)Shamir, Srebro, and Zhang]{pmlr-v32-shamir14}
O.~Shamir, N.~Srebro, and T.~Zhang.
\newblock Communication-efficient distributed optimization using an approximate newton-type method.
\newblock In \emph{Proceedings of the 31st International Conference on Machine Learning}, number~2 in Proceedings of Machine Learning Research, pages 1000--1008, Bejing, China, 22--24 Jun 2014. PMLR.

\bibitem[Sharma and Marchang(2024)]{probabilistic-signflip}
A.~Sharma and N.~Marchang.
\newblock Probabilistic sign flipping attack in federated learning.
\newblock In \emph{2024 15th International Conference on Computing Communication and Networking Technologies (ICCCNT)}, 2024.

\bibitem[Shi et~al.(2022)Shi, Wan, Hu, Lu, and Zhang]{shi2022challenges}
J.~Shi, W.~Wan, S.~Hu, J.~Lu, and L.~Y. Zhang.
\newblock Challenges and approaches for mitigating byzantine attacks in federated learning.
\newblock In \emph{2022 IEEE International Conference on Trust, Security and Privacy in Computing and Communications (TrustCom)}, pages 139--146. IEEE, 2022.

\bibitem[So et~al.(2021)So, Güler, and Avestimehr]{brea2021}
J.~So, B.~Güler, and A.~S. Avestimehr.
\newblock Byzantine-resilient secure federated learning.
\newblock \emph{IEEE Journal on Selected Areas in Communications}, 39\penalty0 (7), 2021.

\bibitem[Wang et~al.(2020)Wang, Liu, Liang, Joshi, and Poor]{fednova}
J.~Wang, Q.~Liu, H.~Liang, G.~Joshi, and H.~V. Poor.
\newblock Tackling the objective inconsistency problem in heterogeneous federated optimization.
\newblock NIPS '20, Red Hook, NY, USA, 2020. Curran Associates Inc.
\newblock ISBN 9781713829546.

\bibitem[Wang et~al.(2019)Wang, Mou, and Sundaram]{wang2019computingTverbergPoint}
X.~Wang, S.~Mou, and S.~Sundaram.
\newblock A resilient convex combination for consensus-based distributed algorithms.
\newblock \emph{Numerical Algebra, Control and Optimization}, 9\penalty0 (3):\penalty0 269--281, 2019.
\newblock ISSN 2155-3289.

\bibitem[Wang et~al.(2021)Wang, Xia, and Zhan]{elite}
Y.~Wang, Y.~Xia, and Y.~Zhan.
\newblock Elite: Defending federated learning against byzantine attacks based on information entropy.
\newblock In \emph{2021 China Automation Congress (CAC)}, pages 6049--6054, 2021.

\bibitem[Wu et~al.(2020)Wu, Ling, Chen, and Giannakis]{variance-reduced-sgd}
Z.~Wu, Q.~Ling, T.~Chen, and G.~B. Giannakis.
\newblock Federated variance-reduced stochastic gradient descent with robustness to byzantine attacks.
\newblock \emph{IEEE Transactions on Signal Processing}, 68:\penalty0 4583--4596, 2020.

\bibitem[Xie et~al.(2019)Xie, Koyejo, and Gupta]{pmlr-v97-xie19b}
C.~Xie, S.~Koyejo, and I.~Gupta.
\newblock Zeno: Distributed stochastic gradient descent with suspicion-based fault-tolerance.
\newblock In \emph{Proceedings of the 36th International Conference on Machine Learning}, volume~97 of \emph{Proceedings of Machine Learning Research}, pages 6893--6901. PMLR, 09--15 Jun 2019.

\bibitem[Xu et~al.(2022)Xu, Huang, Song, and Lan]{gradient-filtering2022}
J.~Xu, S.-L. Huang, L.~Song, and T.~Lan.
\newblock Byzantine-robust federated learning through collaborative malicious gradient filtering.
\newblock In \emph{2022 IEEE 42nd International Conference on Distributed Computing Systems (ICDCS)}, pages 1223--1235, 2022.

\bibitem[Yang and Bajwa(2019)]{Yang_2019}
Z.~Yang and W.~U. Bajwa.
\newblock Byrdie: Byzantine-resilient distributed coordinate descent for decentralized learning.
\newblock \emph{IEEE Transactions on Signal and Information Processing over Networks}, 5\penalty0 (4):\penalty0 611–627, Dec. 2019.
\newblock ISSN 2373-7778.

\bibitem[Yin et~al.(2019)Yin, Malkhi, Reiter, Gueta, and Abraham]{yin2019hotstuff}
M.~Yin, D.~Malkhi, M.~K. Reiter, G.~G. Gueta, and I.~Abraham.
\newblock Hotstuff: Bft consensus with linearity and responsiveness.
\newblock In \emph{Proceedings of the 2019 ACM Symposium on Principles of Distributed Computing}, pages 347--356, 2019.

\bibitem[Zhang et~al.(2021)Zhang, Xie, Bai, Yu, Li, and Gao]{ZHANG2021106775}
C.~Zhang, Y.~Xie, H.~Bai, B.~Yu, W.~Li, and Y.~Gao.
\newblock A survey on federated learning.
\newblock \emph{Knowledge-Based Systems}, 216:\penalty0 106775, 2021.
\newblock ISSN 0950-7051.

\bibitem[Zhang et~al.(2013)Zhang, Duchi, Jordan, and Wainwright]{NIPS2013_d6ef5f7f}
Y.~Zhang, J.~Duchi, M.~I. Jordan, and M.~J. Wainwright.
\newblock Information-theoretic lower bounds for distributed statistical estimation with communication constraints.
\newblock In \emph{Advances in Neural Information Processing Systems}, volume~26. Curran Associates, Inc., 2013.

\bibitem[Zhao et~al.(2018)Zhao, Li, Lai, Suda, Civin, and Chandra]{zhao2018federated}
Y.~Zhao, M.~Li, L.~Lai, N.~Suda, D.~Civin, and V.~Chandra.
\newblock Federated learning with non-iid data.
\newblock \emph{arXiv preprint arXiv:1806.00582}, 2018.

\end{thebibliography}

\newpage


\end{document}